\documentclass[final,12pt]{alt2024} % Anonymized submission
\newcount\Comments
\usepackage{bbm}
\usepackage{mathtools}
\usepackage{algorithm}
\PassOptionsToPackage{algo2e,ruled, linesnumbered}{algorithm2e}
\RequirePackage{algorithm2e}

\usepackage{color}
\definecolor{darkgreen}{rgb}{0,0.5,0}
\definecolor{purple}{rgb}{1,0,1}
\newcommand{\kibitz}[2]{\ifnum\Comments=1\textcolor{#1}{#2}\fi}
\newcommand{\Acal}{\mathcal{A}}

\newcommand{\Hcal}{\mathcal{H}}

\newcommand{\Xcal}{\mathcal{X}}
\newcommand{\Ycal}{\mathcal{Y}}

\DeclareMathOperator*{\argmin}{arg\,min}

\title[Bandit Online Learnability]{Multiclass Online Learnability under Bandit Feedback}
\usepackage{times}
\altauthor{%
 \Name{Ananth Raman} \Email{ananthsraman2007@gmail.com}\\
 \addr Bridgewater, New Jersey
 \AND
 \Name{Vinod Raman}\thanks{Equal contribution} \Email{vkraman@umich.edu}\\
 \addr University of Michigan%
 \AND
 \Name{Unique Subedi}\footnotemark[1]
 \Email{subedi@umich.edu}\\
 \addr University of Michigan%
  \AND
  \Name{Idan Mehalel}\footnotemark[1] \Email{idanmehalel@gmail.com}\\
 \addr Technion%
  \AND
 \Name{Ambuj Tewari} \Email{tewaria@umich.edu}\\
 \addr University of Michigan%
}

\begin{document}

\maketitle

\begin{abstract}%
   We study online multiclass classification under bandit feedback. We extend the results of \cite{daniely2013price} by showing that the finiteness of the Bandit Littlestone dimension is necessary and sufficient for bandit online  learnability even when the label space is unbounded. Moreover, we show that, unlike the full-information setting, sequential uniform convergence is necessary but not sufficient for bandit online learnability. Our result complements the recent work by \cite*{hanneke2023multiclass} who show that the Littlestone dimension characterizes online multiclass learnability in the full-information setting even when the label space is unbounded. %
\end{abstract}


\section{Introduction}
In the standard online multiclass classification model, a learner plays a repeated game against an adversary. In each round $t \in [T]$, an adversary picks a labeled instance $(x_t, y_t) \in \mathcal{X} \times \mathcal{Y}$ and reveals $x_t$ to the learner. Using access to a hypothesis class $\mathcal{H} \subseteq \mathcal{Y}^{\mathcal{X}}$, the learner makes a possibly random prediction $\hat{y}_t \in \mathcal{Y}$. The adversary then reveals the true label $y_t$ and the learner then suffers the loss $\mathbbm{1}\{y_t \neq \hat{y}_t\}$. Overall, the goal of the learner is to output predictions such that its expected cumulative loss is not too much larger than the smallest cumulative loss amongst all fixed hypothesis in $\mathcal{H}$. This standard setting of online multiclass classification is commonly referred to as the \textit{full-information} setting because the learner gets to observe the true label $y_t$ at the end of each round. Perhaps a more practical setting is the \textit{bandit} feedback setting, where the learner does not get to observe the true label at the end of each round, but only the indication $\mathbbm{1}\{\hat{y}_t \neq y_t\}$ of whether its prediction was correct or not \citep*{kakade2008efficient}. One application of this setting is online advertising where the advertiser recommends an ad (label) to a user (instance), but only gets to observe whether the user clicked on the ad or not. 

Unlike the full-information setting, where online learnability of a hypothesis class $\mathcal{H}$ has been fully characterized in both the realizable and agnostic settings, less is known about online learnability under bandit feedback. Indeed, the first work on characterizing bandit online learnability is due to \cite*{DanielyERMprinciple}. They introduce a dimension named the Bandit Littlestone dimension (BLdim), and show that it exactly characterizes the bandit online learnability of deterministic learners in the realizable setting. Even prior to that, \cite{auer1999structural} related the Bandit Littlestone dimension (which is the optimal deterministic mistake bound with bandit feedback) to the multiclass extension of the Littlestone dimension (Ldim) \citep{Littlestone1987LearningQW, DanielyERMprinciple} and showed that $\text{BL}(\mathcal{H}) = O (|\mathcal{Y}|\log(|\mathcal{Y}|)\text{L}(\mathcal{H}))$, where $\text{BL}(\mathcal{H})$ is the BLdim of $\mathcal{H}$, $\text{L}(\mathcal{H})$ is the Ldim of $\mathcal{H}$, and $|\mathcal{Y}|$ denotes the size of the label space $\mathcal{Y}$. Following the work of \cite{DanielyERMprinciple}, \cite{daniely2013price} studies the price of bandit feedback by quantifying the ratio between optimal error rates of the two feedback models in the realizable and agnostic settings.
%Along the way, they relate the Bandit Littlestone dimension to the multiclass extension of the Littlestone dimension (Ldim) \citep{Littlestone1987LearningQW, DanielyERMprinciple} and show that $\text{BL}(\mathcal{H}) \leq 4|\mathcal{Y}|\log(|\mathcal{Y}|)\text{L}(\mathcal{H})$, where $\text{BL}(\mathcal{H})$ is the BLdim of $\mathcal{H}$, $\text{L}(\mathcal{H})$ is the Ldim of $\mathcal{H}$, and $|\mathcal{Y}|$ denotes the size of the label space $\mathcal{Y}$.
Using the inequality $\text{L}(\mathcal{H}) \leq \text{BL}(\mathcal{H}) = O(|\mathcal{Y}|\log(|\mathcal{Y}|)\text{L}(\mathcal{H}))$, they infer that BLdim characterizes realizable online learnability under bandit feedback when $|\mathcal{Y}|$ is finite. Later, \cite{long2017new} and  \cite{geneson2021note} proved that this upperbound on BLdim is the best possible up to a leading constant. They also found the exact optimal leading constant.

Moving beyond the realizable setting, \cite{daniely2013price} give an agnostic online learner whose expected regret, under bandit feedback, is at most $O\left(\sqrt{\text{L}(\mathcal{H})|\mathcal{Y}|T\log(T|\mathcal{Y}|)}\right)$. As a corollary, when $|\mathcal{Y}|$ is finite, they infer that the BLdim qualitatively characterizes  agnostic bandit online learnability.  In addition, \cite{daniely2013price} note a gap of $\tilde{O}(\sqrt{|\mathcal{Y}|})$ between their upperbound in the bandit setting and the known lowerbound of $\Omega(\sqrt{\text{L}(\mathcal{H}) \, T})$ in the full-information setting \citep*{ben2009agnostic}. Accordingly, they ask whether a tighter quantitative characterization of bandit learnability is possible in the agnostic setting. In fact, it is unclear whether BLdim  characterizes bandit online learnability when $|\mathcal{Y}|$ is unbounded. 

Along this direction, there has been a recent surge of interest in characterizing learnability when the label space is unbounded. For example, in a recent breakthrough result, \cite*{brukhim2022characterization} show that the Daniely-Schwartz (DS) dimension, defined by \cite{daniely2014optimal}, characterizes multiclass learnability in the PAC setting even when the label space in unbounded. Following this work, \cite*{hanneke2023multiclass} show that the multiclass extension of the Littlestone dimension, originally proposed by \cite*{DanielyERMprinciple}, continues to characterize online multiclass learnability under \textit{full-information feedback} when the label space is unbounded.  Motivated by these results, we ask whether the BLdim continues to characterize \textit{bandit online learnability} even when the label space is unbounded. In particular, can the optimal expected regret in the realizable and agnostic settings, under bandit feedback, be expressed as a function of the BLdim without a dependence on $|\mathcal{Y}|$? 

In this paper, we resolve this question by showing that the finiteness of BLdim is necessary and sufficient for bandit online learnability, in both the realizable and agnostic settings, even when the label space is unbounded. 

\begin{theorem}\label{thm:necsuff}
    Let $\Hcal \subseteq \Ycal^\Xcal$ and $C_{\mathcal{H}} := \sup_{x\in \mathcal{X}} |\{h(x) : h \in \mathcal{H}\}|$. The following statements are equivalent:
    \begin{enumerate}
        \item $\Hcal$ is bandit online learnable.
        \item $\operatorname{BL}(\Hcal) < \infty$.
        \item $C_{\mathcal{H}} < \infty$ and $\operatorname{L}(\Hcal) < \infty$.
    \end{enumerate}
\end{theorem}

We prove $(2) \implies (1)$, $(3) \implies (2)$ in Section~\ref{sec:Bldimsuff}, and $(1) \implies (3)$ in Section~\ref{sec:necc}.
The proof of $(2) \implies (1)$ is given by an agnostic online learner whose expected regret under bandit feedback can be expressed as a function of BLdim without any dependence on $|\mathcal{Y}|$. 

\begin{theorem} \label{thm:banditagn}
    For any $\mathcal{H} \subseteq \mathcal{Y}^{\mathcal{X}}$, there exists an agnostic online learner  whose expected regret, under bandit feedback, is at most
$$8\sqrt{\emph{\text{L}}(\mathcal{H})\emph{\text{BL}}(\mathcal{H})T \log(T)}.$$
\end{theorem}

 Theorem \ref{thm:banditagn} provides an improvement over the upperbound given by \cite{daniely2013price} when $|\mathcal{Y}| \gg \text{BL}(\mathcal{H})$. In fact, the gap between $|\mathcal{Y}|$ and $\text{BL}(\mathcal{H})$ can be arbitrary.  Consider the case where $\mathcal{Y} = \mathbb{N}$ but $|\mathcal{H}| < \infty$. Then, its not hard to see that $\text{BL}(\mathcal{H}) \leq |\mathcal{H}|$ but $|\mathcal{Y}| = \infty$.  

In addition to characterizing learnability, there has been recent interest in showing a separation between uniform convergence and learnability. For example, \cite*{montasser2019vc} show that while uniform convergence is sufficient for adverarsially robust PAC learnability, it is not necessary. Likewise, for online mutliclass learning under full-information feedback, \cite{hanneke2023multiclass} give a class that is learnable, but the online analog of uniform convergence \citep*{rakhlin2015sequential}, termed Sequential Uniform Convergence (SUC), does not hold. Towards this end, we ask whether there is a separation between SUC and bandit online learnability.  We answer this question affirmatively: while SUC is \textit{necessary} for bandit learnability, it is not sufficient. 

 \begin{theorem} \label{thm:banditUC}
 If a hypothesis class is online learnable under bandit feedback, then it enjoys the SUC property. However, there exists a class which satisfies the SUC property, but is not online learnable under bandit feedback. 
 \end{theorem}

 Theorem \ref{thm:banditUC} is in contrast to the full information setting where SUC is sufficient, but not necessary for online learnability \citep{hanneke2023multiclass}. We note that Theorem \ref{thm:banditUC} along with Example 1 from \cite{hanneke2023multiclass} also shows a separation in online learnability between the full-information and bandit feedback settings. Figure \ref{fig:relation} visualizes the landscape of learnability for online multiclass problems.
 % \idan{The figure contains the SG dimension (in an equivalent notation), but this dimension is only defined later. I think the notation in the figure should be changed to SG, instead of $\operatorname{L}(\ell \cdot \mathcal{H})$ (which is even more vague for readers who did not read the definition later on), and the figure's caption should refer to the definition of SG. I would do it myself but I don't have access to the figure's source.}
 \begin{figure}[htp]
    \centering
    \includegraphics[width=5cm]{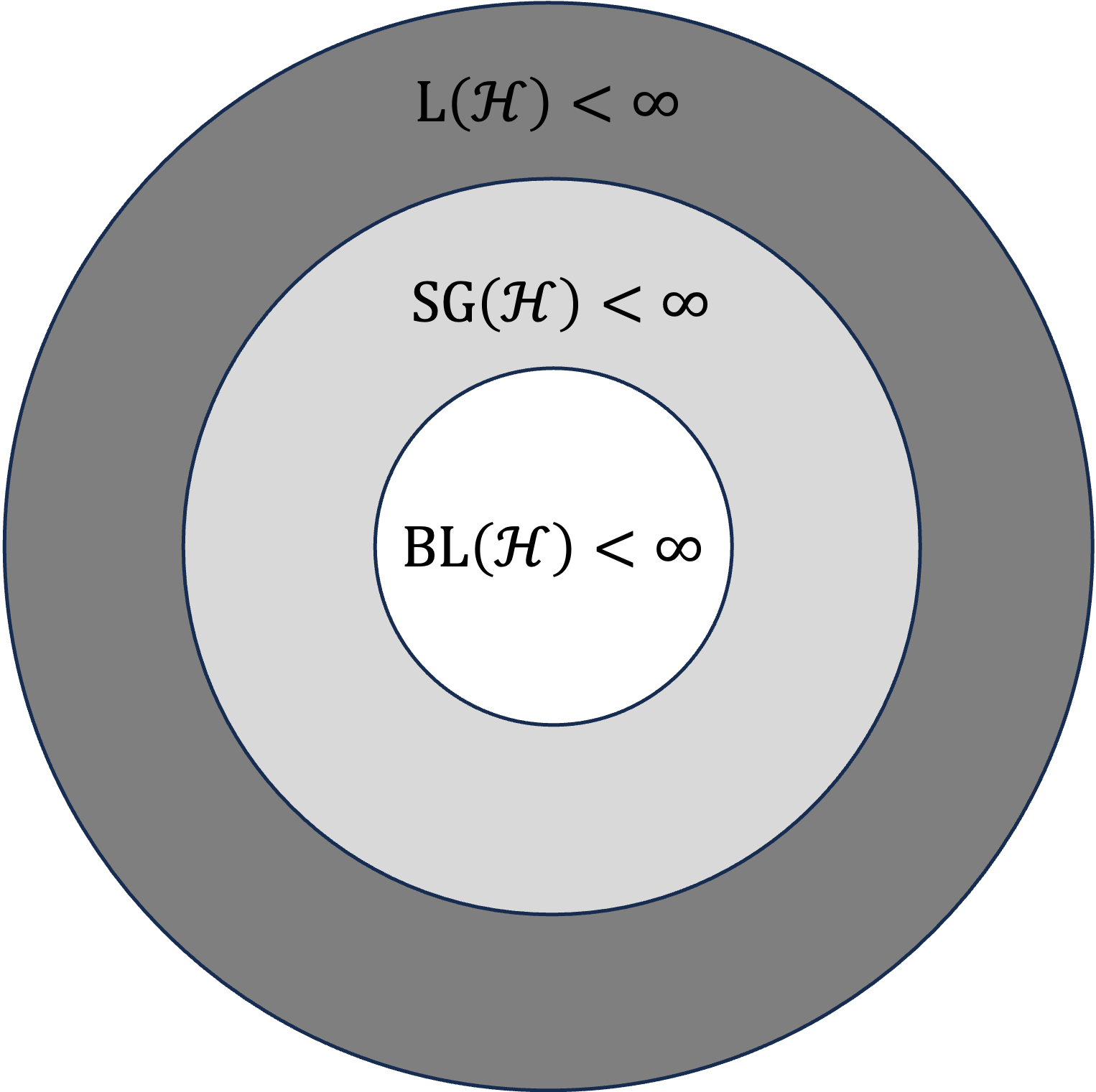}
    \caption{Landscape of multiclass online learnability. The Sequential Graph (SG) dimension (see Definition \ref{def:sgdim}) characterizes SUC.}
    \label{fig:relation}
\end{figure}

\section{Preliminaries}

Let $\Xcal$ denote the instance space, $\Ycal$ be the label space, and $\mathcal{H} \subseteq \mathcal{Y}^{\mathcal{X}}$ denote a hypothesis class. In this paper, we place no assumptions on the size of the label space  $\mathcal{Y}$. 
% \idan{Are we completely sure? I recall one reviewer had a problem with that.. Even though it looks fine to me}
% Accordingly, we will assume that there exists a natural ordering on $\mathcal{Y}$. Let $\mathcal{H} \subseteq \mathcal{Y}^{\mathcal{X}}$ denote a hypothesis class consisting of predictors $h: \mathcal{X} \rightarrow \mathcal{Y}$. 
Given an instance $x \in \mathcal{X}$, we let $\mathcal{H}(x) := \{h(x): h \in \mathcal{H}\}$ denote the projection of $\mathcal{H}$ onto $x$. 
% Given an index $i \in \{1, ..., |\mathcal{H}(x)|\}$, we let $\mathcal{H}_i(x)$ denote the $i$'th element of $\mathcal{H}(x)$ after sorting $\mathcal{H}(x) \subseteq \mathcal{Y}$ in its natural order.
As usual, $[N]$ is used to denote $\{1, 2, \ldots, N\}$. 
% Given a distribution $p$ over $\mathcal{Y}$, we let $\text{supp}(p) \subseteq \mathcal{Y}$ denote its support. 

\subsection{Online Learning} 
In online multiclass classification with bandit feedback, an adversary plays a sequential game with the learner over $T$ rounds. In each round $t \in [T]$, an adversary selects a labeled instance $(x_t, y_t) \in \mathcal{X} \times \mathcal{Y}$ and reveals $x_t$ to the learner. The learner makes a (potentially randomized) prediction $\hat{y}_t \in \mathcal{Y}$. Finally, the adversary reveals to the learner its loss $\mathbbm{1}\{\hat{y}_t \neq y_t\}$, but not the true label $y_t$. Given a hypothesis class  $\mathcal{H} \subseteq \mathcal{Y}^{\mathcal{X}}$, the goal of the learner is to output predictions $\hat{y}_t$ under \textit{bandit feedback} such that its \textit{expected regret}

$$\mathbb{E}\left[\sum_{t=1}^T \mathbbm{1}\{\hat{y}_t \neq y_t\} - \inf_{h \in \mathcal{H}}\sum_{t=1}^T  \mathbbm{1}\{h(x_t) \neq  y_t \}\right] $$
is small. A hypothesis class $\mathcal{H}$ is said to be bandit online learnable if there exists an algorithm such that for any sequence of labeled examples $(x_1, y_1), ..., (x_T, y_T)$,  its expected regret, under bandit feedback, is a sublinear function of $T$. In this paper, we consider the oblivious setting where the adversary selects the entire sequence of labeled instances $(x_1, y_1), ..., (x_T, y_T)$ before the game begins. Thus, we treat the stream of labeled instances as a non-random, deterministic quantity.

\begin{definition} [Bandit Online Learnability]
\label{def:agnOL}
A hypothesis class $\Hcal$ is bandit online learnable, if there exists an (potentially randomized) algorithm $\mathcal{A}$ such that its \emph{expected regret}, 

$$R_{\mathcal{A}}(T, \mathcal{H}) := \sup_{(x_1, y_1), ..., (x_T, y_T)} \left(\mathbb{E}\left[\sum_{t=1}^T \mathbbm{1}\{\mathcal{A}(x_t) \neq y_t\}\right] - \inf_{h \in \mathcal{H}}\sum_{t=1}^T  \mathbbm{1}\{h(x_t) \neq  y_t \}\right),$$

\noindent while only receiving \emph{bandit feedback}, is a non-decreasing sub-linear function of $T$. 
\end{definition}

If it is guaranteed that the learner always observes a sequence of examples labeled by some hypothesis $h \in \mathcal{H}$, then we say that we are in the \textit{realizable} setting. 

\cite{Littlestone1987LearningQW} and \cite*{ben2009agnostic} showed that a combinatorial parameter called the Littlestone dimension characterizes online learnability of binary hypothesis classes under full-information feedback, in both the realizable and agnostic settings, respectively. Later,  \cite{DanielyERMprinciple} defined a multiclass extension of the Littlestone dimension and showed that it tightly characterizes online learnability of multiclass hypothesis classes under full-information feedback in both the realizable and agnostic settings. The Littlestone dimension, in both the binary and multiclass case, is defined in terms of trees, a combinatorial object that captures the temporal dependence inherent in online learning. 

% \cite{DanielyERMprinciple} and \cite{daniely2013price} showed that a combinatorial measure called the Bandit Littlestone dimension (BLdim) characterizes the bandit online learnability of determinsitic learners in the realizable setting. Like many other combinatorial dimensions in online learning, the BLdim is defined using trees. Accordingly, in order to define the BLdim, we need to first rigorously define a tree. Our notation here is borrowed from \cite{rakhlin2015online} and \cite{raman2023online}.

 Given an instance space $\mathcal{X}$ and a set of objects $\mathcal{M}$, an $\mathcal{X}$-valued, $\mathcal{M}$-ary tree $\mathcal{T}$ of depth $T$ is a complete rooted tree such that (1) each internal node $v$ is labeled by an instance $x \in \mathcal{X}$ and (2) for every internal node $v$ and object $m \in \mathcal{M}$, there is an outgoing edge $e^m_{v}$  indexed by $m$. Such a tree can be identified by a sequence $(\mathcal{T}_1, ..., \mathcal{T}_T)$ of labeling functions $\mathcal{T}_t:\mathcal{M}^{t-1} \rightarrow \mathcal{X}$ which provide the labels for each internal node. A path of length $T$ is given by a sequence of objects $m = (m_1,..., m_T) \in \mathcal{M}^T$. Then, $\mathcal{T}_t(m_1, ..., m_{t-1})$ gives the label of the node by following the path $(m_1, ..., m_{t-1})$ starting from the root node, going down the edges indexed by the $m_t$'s.  We let $\mathcal{T}_1 \in \mathcal{X}$ denote the instance labeling the root node. For brevity, we define $m_{<t} = (m_1, ..., m_{t-1})$ and therefore write $\mathcal{T}_t(m_1, ..., m_{t-1}) = \mathcal{T}_t(m_{<t})$. Analogously, we let $m_{\leq t} = (m_1, ..., m_{t})$.
 % Often, it is useful to label the edges of a tree with some \textit{auxiliary} information. Given an $\mathcal{X}$-valued, $\mathcal{M}$-ary tree $\mathcal{T}$ of depth $T$ and a (potentially uncountable) set of objects $\mathcal{N}$, we can formally label the edges of $\mathcal{T}$ using objects in $\mathcal{N}$ by considering a sequence $(f_1, ..., f_T)$ of edge-labeling functions  $f_t: \mathcal{M}^{t} \rightarrow \mathcal{N}$. For each depth $t \in [T]$, the function $f_t$ takes as input a path $m_{\leq t}$ of length $t$ and outputs an object in $\mathcal{N}$. Accordingly, we can think of the object $f_t(m_{\leq t})$ as labeling the edge indexed by $m_t$ after following the path $m_{< t}$ down the tree. 
 Using this notation, we define the extension of the Littlestone dimension to the multiclass setting proposed by \cite{DanielyERMprinciple}.

\begin{definition}[Littlestone dimension \citep{Littlestone1987LearningQW, DanielyERMprinciple}]
Let $\mathcal{T}$ be a complete, $\mathcal{X}$-valued, $\{\pm 1\}$-ary tree of depth $d$ such that the edges from a single parent node to its child nodes are each labeled with a different element of $\mathcal{Y}$. The tree $\mathcal{T}$ is shattered by $\mathcal{H} \subseteq \Ycal^{\Xcal}$  if for every path $\sigma = (\sigma_1, ..., \sigma_d) \in \{\pm 1\}^d$, there exists a hypothesis $h_{\sigma} \in \mathcal{H}$ such that for all $t \in [d]$,  $h_{\sigma}(\mathcal{T}_t(\sigma_{<t})) = y(\sigma_{\leq t})$, where $y(\sigma_{\leq t})$ is the label of the edge between the the nodes $(\mathcal{T}_t(\sigma_{<t}), (\mathcal{T}_{t+1}(\sigma_{\leq t})))$.  The Littlestone dimension of $\mathcal{H}$, denoted $\emph{\text{L}}(\mathcal{H})$, is the maximal depth of a tree $\mathcal{T}$ that is shattered by $\mathcal{H}$. If there exist shattered trees of arbitrarily large depth, we say that $\emph{\text{L}}(\mathcal{H}) = \infty$.
\end{definition}

 In the same work, \cite{DanielyERMprinciple} defined a combinatorial parameter called the Bandit Littlestone dimension (BLdim) and showed that it characterizes bandit online learnability of deterministic learners in the realizable setting. 

\begin{definition}[Bandit Littlestone dimension \citep{DanielyERMprinciple}] \label{def:bldim}
Let $\mathcal{T}$ be a complete, $\mathcal{X}$-valued, $\mathcal{Y}$-ary tree of depth $d$. The tree $\mathcal{T}$ is shattered by $\mathcal{H} \subseteq \Ycal^{\Xcal}$  if for every path $y = (y_1, ..., y_d) \in \mathcal{Y}^d$, there exists a hypothesis $h_{y} \in \mathcal{H}$ such that for all $t \in [d]$,  $h_{y}(\mathcal{T}_t(y_{<t})) \neq y_t$. The Bandit Littlestone dimension of $\mathcal{H}$, denoted $\emph{\text{BL}}(\mathcal{H})$, is the maximal depth of a tree $\mathcal{T}$ that is shattered by $\mathcal{H}$. If there exist shattered trees of arbitrarily large depth, we say that $\emph{\text{BL}}(\mathcal{H}) = \infty$.
\end{definition}

% We make a few important remarks about the BLdim. First, we note that while our notation for the BLdim differs slightly from that used by \cite{DanielyERMprinciple} and \cite{daniely2013price}, the two definitions are equivalent. The most important fact is that a path down a BLdim tree of depth $d$ is defined by a sequence of labels in $\mathcal{Y}$ of length $d$ and that a tree $\mathcal{T}$ is shattered by a hypothesis $h \in \mathcal{H}$ if its outputs on the instances along the path do not match the corresponding labels in the path. Second, note that the arity of the BLdim tree depends on the size of $\mathcal{Y}$. Indeed, observe that the BLdim tree has an edge for every element of $\mathcal{Y}$. Accordingly, if, say, $\mathcal{Y} = \mathbb{N}$, then the arity of the BLdim is infinite. 

% Before we give the definition, it will be useful to define that a binary tree $\mathcal{T}$ is $\mathcal{Z}$-valued if its internal nodes are labelled by elements from $\mathcal{Z}$.

In particular, \cite{DanielyERMprinciple} show a matching upper and lowerbound on the realizable error rate of deterministic learners in terms of the BLdim. 

\begin{theorem} [Realizable Learnability \citep{DanielyERMprinciple}] \label{thm:reallearn}
    In the realizable setting, for any $\mathcal{H} \subseteq \mathcal{Y}^{\mathcal{X}}$, there exists a deterministic online learner whose cumulative loss on the worst-case sequence, under bandit feedback, is at most $\emph{\text{BL}}(\mathcal{H})$.  Also, the cumulative loss of any deterministic online learner on the worst-case sequence, under bandit feedback, is at least $\emph{\text{BL}}(\mathcal{H})$. 
\end{theorem}

 In the agnostic setting, \cite{daniely2013price} gave an upperbound on the expected regret under bandit feedback, in terms of $|\mathcal{Y}|$ and Ldim. 

% and showed that it \textit{can} be tight. 

\begin{theorem} [Agnostic Learnability \citep{daniely2013price}] \label{thm:finitek}
    For any $\mathcal{H} \subseteq \mathcal{Y}^{\mathcal{X}}$, there exists an online learner $\mathcal{A}$ such that

    $$R_{\mathcal{A}}(T, \mathcal{H}) \leq e\sqrt{\emph{\text{L}}(\mathcal{H})|\mathcal{Y}|T\log(T|\mathcal{Y}|)}.$$
    % Moreover, \emph{there exists} a hypothesis class $\mathcal{H} \subseteq \mathcal{Y}^{\mathcal{X}}$ such that for any (randomized) online learner, there exists a stream such that the expected regret of the learner, under bandit feedback, is at least 
    
    % $$\frac{1}{20}\sqrt{T|\mathcal{Y}|\emph{\text{L}}(\mathcal{H})}.$$

\end{theorem}

In Section \ref{sec:Bldimsuff}, we show that $|\mathcal{Y}|$ in Theorem~\ref{thm:finitek} can be replaced with $\operatorname{BL}(\Hcal)$. Note that this both \emph{qualitatively} and \emph{quantitatively} improves over Theorem~\ref{thm:finitek}. Qualitatively, it shows that finite $\operatorname{BL}(\Hcal)$ suffices for bandit online learnability, without any requirements on $|\Ycal|$. Furthermore, there is no better qualitative characterization, since as we show in Section~\ref{sec:necc}, finite $\operatorname{BL}(\Hcal)$ is also \emph{necessary} for learnability. A quantitative improvement is achieved in cases where $|\Ycal| \gg \operatorname{BL}(\Hcal)$.

\subsection{Online Learnability and Uniform Convergence}

The relationship between learnability and uniform convergence has a rich history in learning theory. For binary classification in the PAC setting, the seminal work by \cite{vapnik1974theory} shows that  uniform convergence and PAC learnability are equivalent. Likewise, for online binary classification, an online analog of uniform convergence, termed Sequential Uniform Convergence (SUC), is equivalent to online learnability \citep*{rakhlin2015sequential, alon2021adversarial}. However, this equivalence between uniform convergence and learnability breaks down for multiclass classification. Indeed, in the PAC setting, it was shown that while uniform convergence suffices for multiclass learnability, it is not necessary \citep*{Natarajan1989}. Recently, \cite{hanneke2023multiclass} extended this separation to the online, full-information feedback setting by showing  that SUC is sufficient but not necessary for multiclass learnability. Instead, they show that SUC is characterized by a different combinatorial parameter termed the Sequential Graph dimension (SGdim).

\begin{definition} [Sequential Graph dimension \citep{hanneke2023multiclass}] \label{def:sgdim}
    Let $\mathcal{H} \subseteq \mathcal{Y}^{\mathcal{X}}$ and $\ell \circ \mathcal{H} = \{(x, y) \mapsto \mathbbm{1}\{h(x) \neq y\}: h \in \mathcal{H}\}$ be its loss class. Then, the Sequential Graph dimension of $\mathcal{H}$, denoted $\emph{\text{SG}}(\mathcal{H})$, is defined as $\emph{\text{SG}}(\mathcal{H}) = \emph{\text{L}}(\ell \circ \mathcal{H})$.
\end{definition}
In particular, a hypothesis class $\mathcal{H}$ enjoys the SUC property if and only if its SGdim is finite. 

\begin{theorem}[\cite{hanneke2023multiclass, rakhlin2015online}]\label{prop:sguc}
For any hypothesis class $\mathcal{H} \subseteq \mathcal{Y}^{\mathcal{X}}$,  the SUC property holds for $\mathcal{H}$ if and only if $\emph{\text{SG}}(\mathcal{H}) < \infty.$
\end{theorem}
In fact, there is a quantitative relation between SGdim and Ldim when the label space is bounded. 

\begin{theorem} [\cite{hanneke2023multiclass}] \label{prop:sgldim} 
For any $\mathcal{H} \subseteq \mathcal{Y}^{\mathcal{X}}$ such that $|\mathcal{Y}| < \infty$, we have $\emph{\text{SG}}(\mathcal{H}) = O(\emph{\text{L}}(\mathcal{H})\log(|\mathcal{Y}|)).$
\end{theorem}

A combination of Theorem~\ref{prop:sgldim} and a result due to \cite{alon2021adversarial}, \cite{hanneke2023multiclass} derives a new upperbound on the best achievable expected regret under full-information feedback.

\begin{theorem} [\cite{hanneke2023multiclass, alon2021adversarial}] \label{prop:sgrates} 
For any $\mathcal{H} \subseteq \mathcal{Y}^{\mathcal{X}}$ such that $\emph{\text{SG}}(\mathcal{H}) < \infty$, there exists an online learner whose expected regret under full-information feedback is at most  $O(\sqrt{\emph{\text{SG}}(\mathcal{H})\, T}).$
\end{theorem}

In this work, we also investigate the relationship between bandit online learnablity and SUC. In Section \ref{sec:Bldimsuff}, we show that finite BLdim implies finite SGdim, and more precisely that $\text{SG}(\mathcal{H}) = O(\text{L}(\mathcal{H})\log(\text{BL}(\mathcal{H}))).$ On the other hand, in Section \ref{sec:necc}, we exhibit a class where SUC holds, but is not bandit online learnable. Together, these results imply that SUC is necessary, but not sufficient, for bandit online learnability. 

\section{BLdim is Sufficient for Bandit Online Learnability} \label{sec:Bldimsuff}

In this section we prove Theorem~\ref{thm:banditagn}, which implies direction $(2) \implies (1)$ in Theorem~\ref{thm:necsuff}. We also prove $(3) \implies (2)$ towards the end of the section. The first ingredient of this proof is the following result which shows that the BLdim provides a uniform upperbound on the size of the projection of $\mathcal{H}$ on any instance $x \in \mathcal{X}$.

\begin{lemma} \label{lem:proj} For any $\mathcal{H} \subseteq \mathcal{Y}^{\mathcal{X}}$, we have $\sup_{x \in \mathcal{X}}|\mathcal{H}(x)| \leq \emph{\text{BL}}(\mathcal{H}) + 1$. 
\end{lemma}
\begin{proof}
    Suppose that $|\mathcal{H}(x)| \geq \text{BL}(\mathcal{H})+2$ for some $x \in \mathcal{X}$. We will prove the lemma by contradiction, constructing a BLdim tree of depth $\text{BL}(\mathcal{H}) + 1$ that is shattered by $\mathcal{H}$. Let $\mathcal{T}$ be a BLdim tree of depth $\text{BL}(\mathcal{H})+1$ with every internal node labeled by $x$. Without loss of generality, suppose that $\mathcal{H}(x) = \{1, 2, ..., \text{BL}(\mathcal{H})+2\}$, and let there be $h_1, ..., h_{\text{BL}(\mathcal{H})+2} \in \mathcal{H}$ such that $h_{i}(x) = i$ for all $1 \leq i \leq \text{BL}(\mathcal{H})+2$. We now show that $\mathcal{H}$ shatters $\mathcal{T}$. Consider any path down $\mathcal{T}$. Since $\mathcal{T}$ has depth $\text{BL}(\mathcal{H}) + 1$, there can only be $\text{BL}(\mathcal{H})+1$ different labels on that path. Since there are at least  $\text{BL}(\mathcal{H})+2$ hypotheses in $\mathcal{H}$, there is a hypothesis $h_{i} \in \mathcal{H}$ such that $h_{i}(x)$ is not equal to any of the labels on the path. Since the path is arbitrary, the tree is shattered by $\mathcal{H}$ according to Definition \ref{def:bldim}. By contradiction, $|\mathcal{H}(x)| \leq \text{BL}(\mathcal{H})+1$ for all $x \in \mathcal{X}$.
\end{proof}

A uniform upperbound $C$ on the projection size of $\mathcal{H}$ is a strong property: it allows us to effectively reduce the label space from $\mathcal{Y}$ to $[C]$. Lemma \ref{lem:simplify} makes this precise. For a bandit algorithm $\Acal$, let $\Acal(x)$ be its prediction on $x$, given the history of the game so far (for the sake of readability, we omit the information received prior to instance $x$ from the notation).

% \idan{Changed the property order a bit since property (iv) follows from property (iii). I only had to add to property (iii) the trivial fact that if $\bar{\Acal}$ is deterministic, then $\Acal$ is deterministic as well.}
\begin{lemma} \label{lem:simplify}
Let $\mathcal{H} \subseteq \mathcal{Y}^{\mathcal{X}}$ such that $\sup_{x \in \mathcal{X}}|\mathcal{H}(x)| \leq C$. Then, there exists a hypothesis class $\bar{\mathcal{H}} \subseteq [C]^{\mathcal{X}}$ such that 
\begin{itemize}
    \item[\emph{(i)}] $\emph{\text{L}}(\bar{\mathcal{H}}) = \emph{\text{L}}(\mathcal{H}).$
    \item[\emph{(ii)}] $\emph{\text{SG}}(\bar{\mathcal{H}}) = \emph{\text{SG}}(\mathcal{H}).$
    \item[\emph{(iii)}] For every bandit algorithm $\bar{\mathcal{A}}$ for $\bar{\mathcal{H}}$ such that $\bar{\mathcal{A}}(x) \in \bar{\mathcal{H}}(x)$ at all times, there exists a bandit algorithm $\mathcal{A}$ for $\mathcal{H}$ such that $R_{\mathcal{A}}(T, \mathcal{H}) = R_{\mathcal{\bar{A}}}(T, \bar{\mathcal{H}})$ for all $T$. Furthermore, if $\bar{\mathcal{A}}$ is deterministic, then so is $\Acal$.
    \item[\emph{(iv)}] $\emph{\text{BL}}(\bar{\mathcal{H}}) = \emph{\text{BL}}(\mathcal{H}).$
\end{itemize}
\end{lemma}

\begin{proof}
    Let $\mathcal{H} \subseteq \mathcal{Y}^{\mathcal{X}}$ be such that $\sup_{x \in \mathcal{X}}|\mathcal{H}(x)| \leq C$. For every $x \in \mathcal{X}$, define a function $\phi_x: \mathcal{H}(x) \rightarrow [|\mathcal{H}(x)|]$ such that $\phi_x$ is one-to-one. Finally, consider the following hypothesis class $\bar{\mathcal{H}} = \{x \mapsto \phi_x(h(x)): h \in \mathcal{H}\}$. Clearly, we have that $\bar{\mathcal{H}} \subseteq [C]^{\mathcal{X}}$ and we now show that $\bar{\mathcal{H}}$ also satisfies the four properties above.

    Property (i) follows from observing that any non-empty shattered Ldim tree for $\bar{\mathcal{H}}$ can be transformed into a shattered Ldim tree for $\mathcal{H}$, since the the outgoing edges of any internal node labeled by $x$ must be labeled using elements of $[|\mathcal{H}(x)|]$. Thus, the inverse mapping $\phi^{-1}_x: [|\mathcal{H}(x)|] \rightarrow \mathcal{H}(x)$ can be used to transform this tree into an Ldim tree of the same depth for $\mathcal{H}$. Likewise, one can use the forward mapping $\phi_x$ to transform any non-empty shattered Ldim tree for $\mathcal{H}$ into a non-empty shattered Ldim tree for $\bar{\mathcal{H}}$. The equality trivially holds if no non-empty Ldim tree exists. 
    
    Property (ii) follows from the fact that every internal node in a non-empty shattered Ldim tree for the loss class $\{(x, y) \mapsto \mathbbm{1}\{h(x) \neq y\}: h \in \mathcal{H}\}$ must be labeled using elements of $\{(x, y): y \in \mathcal{H}(x)\}$. Thus the mapping function $\phi_x$ can be used to transform any non-empty shattered Ldim tree for the loss class $\{(x, y) \mapsto \mathbbm{1}\{h(x) \neq y\}: h \in \mathcal{H}\}$ into a non-empty shattered Ldim tree for the loss class $\{(x, y) \mapsto \mathbbm{1}\{\bar{h}(x) \neq y\}: \bar{h} \in \bar{\mathcal{H}}\}$. The reverse direction follows analogously by using the inverse mapping function $\phi_x^{-1}$.

    To prove property (iii), suppose that $\bar{\mathcal{A}}$ is a bandit algorithm for $\bar{\mathcal{H}}$ such that on any instance $x \in \mathcal{X}$, the prediction of $\bar{\mathcal{A}}$ always lies in $\bar{\mathcal{H}}(x)$. Algorithm \ref{alg:reduction} uses $\bar{\mathcal{A}}$ in a black-box fashion to construct a bandit learner $\mathcal{A}$ for $\mathcal{H}$. 

    % To prove Property (iv), suppose that $\bar{\mathcal{A}}$ is a bandit algorithm for $\bar{\mathcal{H}}$ such that on any instance $x \in \mathcal{X}$, $\bar{\mathcal{A}}$ always predicts in $\bar{\mathcal{H}}$. Moreover, let $R_{\bar{\mathcal{A}}}(T, \bar{\mathcal{H}})$ denote the expected regret of $\bar{\mathcal{A}}$. Algorithm \ref{alg:reduction} uses $\bar{\mathcal{A}}$ in a black-box fashion to construct a bandit learner $\mathcal{A}$ for $\mathcal{H}$. 

        \begin{algorithm}
        \caption{Bandit algorithm $\mathcal{A}$}
        \label{alg:reduction}
        \setcounter{AlgoLine}{0}
        \KwIn{Hypothesis class $\mathcal{H}$, bandit algorithm $\bar{\mathcal{A}}$ for $\bar{\mathcal{H}}$}
                
        \For{$t = 1,...,T$} {
            Receive example $x_t$
            
            Query $\bar{y}_t = \bar{\mathcal{A}}(x_t)$

            Predict $\hat{y}_t = \phi_{x_t}^{-1}(\bar{y}_t)$

            Observe loss $\mathbbm{1}\{y_t \neq \hat{y}_t\}$ and pass along the indication to $\bar{\mathcal{A}}$        
         
            }
        \end{algorithm}

 We claim that the expected regret of Algorithm $\mathcal{A}$ is $R_{\bar{\mathcal{A}}}(T, \bar{\mathcal{H}})$. To see this,  fix $T \in \mathbbm{N}$ and let $S = (x_1, y_1), \dots, (x_T,y_T) \in (\mathcal{X} \times \mathcal{Y})^T$ be the  sequence of examples to be passed to $\mathcal{A}$. We show that there exists a sequence of examples $S^{\prime} \in (\mathcal{X} \times [C] \cup \{\star\})^T$ for $\bar{\mathcal{A}}$  such that
\begin{equation} \label{eq:reduc-1}
\min_{h \in \mathcal{H}} \sum_{(x_t,y_t) \in S} \mathbbm{1}\{h(x_t) \neq y_t\}
=
\min_{\bar{h} \in \bar{\mathcal{H}}} \sum_{(x_t,y'_t) \in S'} \mathbbm{1}\{\bar{h}(x_t) \neq y'_t\},
\end{equation}
\begin{equation}\label{eq:reduc-2}
\mathbb{E} \left[\sum_{(x_t,y_t) \in S} \mathbbm{1}\{\mathcal{A}(x_t) \neq y_t\} \right] 
=
\mathbb{E} \left[\sum_{(x_t,y'_t) \in S'} \mathbbm{1}\{\bar{\mathcal{A}}(x_t) \neq y'_t\} \right],
\end{equation}
% \idan{changed property (3) to be more accurate. It is not clear what is the ``correct feedback". Correct with respect to what? The idea is that evaluating $\Acal$ on $S$ exactly simulates evaluating  $\bar{\Acal}$ on $S'$}
and (3) the feedback that $\Acal$ provides to $\bar{\mathcal{A}}$ matches the feedback that $\bar{\mathcal{A}}$ would have received if it was executed on $S'$.
%$\bar{\mathcal{A}}$ receives the correct feedback when evaluated on the stream $S^{\prime}$. 

Combining \eqref{eq:reduc-1}, \eqref{eq:reduc-2}, and (3) and the regret guarantee $R_{\bar{\mathcal{A}}}(T, \bar{\mathcal{H}})$ for $\bar{\mathcal{A}}$ immediately implies property (iii). It remains to construct $S^{\prime}$ for which all three statements hold. For every $t \in [T]$, let $y_t^{\prime} = \phi_{x_t}(y_t) \mathbbm{1}\{y_t \in \mathcal{H}(x_t)\} + \star \mathbbm{1}\{y_t \notin \mathcal{H}(x_t)\}$. Consider the following stream $S^{\prime} = (x_1, y_1^{\prime}), ..., (x_T, y_T^{\prime}) \in (\mathcal{X} \times [C] \cup \{\star\})^T$. To see that \eqref{eq:reduc-1} holds, observe that for every $h \in \mathcal{H}$ we have that
% \idan{changed the equation a bit to match (1)}

\begin{align*}
\sum_{t=1}^T \mathbbm{1}\{h(x_t) \neq y_t\} &=  \sum_{t: y_t \in \mathcal{H}(x_t)} \mathbbm{1}\{h(x_t) \neq y_t\} + \sum_{t: y_t \notin \mathcal{H}(x_t)} \mathbbm{1}\{h(x_t) \neq y_t\}\\
&= \sum_{t: y_t \in \mathcal{H}(x_t)} \mathbbm{1}\{\phi_{x_t}(h(x_t)) \neq \phi_{x_t}(y_t)\} + \sum_{t: y_t \notin \mathcal{H}(x_t)} \mathbbm{1}\{\phi_{x_t}(h(x_t)) \neq \star\}\\
&= \sum_{t=1}^T \mathbbm{1}\{\bar{h}(x_t) \neq y^{\prime}_t\}.
\end{align*}

To see \eqref{eq:reduc-2}, note that
% \idan{changed the third row a bit}

    \begin{align*}
        \mathbb{E}\left[\sum_{t=1}^T \mathbbm{1}\{\mathcal{A}(x_t) \neq y_t\} \right] &= \mathbb{E}\left[\sum_{t= 1}^T \mathbbm{1}\{\phi_{x_t}^{-1}(\bar{y}_t) \neq y_t\}\right]\\
        &= \mathbb{E}\left[\sum_{t: y_t \in \mathcal{H}(x_t)} \mathbbm{1}\{\phi_{x_t}^{-1}(\bar{y}_t) \neq y_t\} + \sum_{t: y_t \notin \mathcal{H}(x_t)} \mathbbm{1}\{\phi_{x_t}^{-1}(\bar{y}_t) \neq y_t\}\right]\\
        &= \mathbb{E}\left[\sum_{t: y_t \in \mathcal{H}(x_t)} \mathbbm{1}\{\bar{\Acal}(x_t) \neq \phi_{x_t}(y_t)\} + \sum_{t: y_t \notin \mathcal{H}(x_t)} \mathbbm{1}\{\bar{\Acal}(x_t) \neq \star\}\right]\\
        &= \mathbb{E}\left[\sum_{t = 1}^T \mathbbm{1}\{\bar{\Acal}(x_t) \neq y^{\prime}_t\}\right].
    \end{align*}

Finally, to prove (3), it suffices to show that $\mathbbm{1}\{y_t \neq \hat{y}_t\} = \mathbbm{1}\{y^{\prime}_t \neq \bar{\Acal}(x_t)\}$. If $\mathbbm{1}\{y_t \neq \hat{y}_t\} = 0$, then $y_t = \phi_{x_t}^{-1}(\bar{y}_t)$ and $\bar{\mathcal{A}}(x_t) = \phi_{x_t}(y_t) = y^{\prime}_t$ as needed. If  $\mathbbm{1}\{y_t \neq \hat{y}_t\} = 1$ and $y_t \in \mathcal{H}(x_t)$, then $\bar{y}_t \neq \phi_{x_t}(y_t) = y^{\prime}_t$. Lastly, if $\mathbbm{1}\{y_t \neq \hat{y}_t\} = 1$ and $y_t \notin \mathcal{H}(x_t)$, we get that $\bar{y}_t \neq y^{\prime}_t = \star$ since $\bar{\mathcal{A}}(x_t) \in \bar{\mathcal{H}}(x_t)$. The ``furthermore" part of the property is straightforward by the construction of $\Acal$.
% We partition the examples of $S$ into two sets: \emph{good} examples and \emph{bad} examples. An example $(x_t,y_t)$ is bad if $y_t \notin \mathcal{H}(x_t)$. All other examples are good. For the sake of readability, suppose without loss of generality that all bad examples are at the last places of $S$.  We now construct $S'$ from $S$ in two simple steps:
% \begin{enumerate}
%     \item Remove all bad examples from $S$.
%     \item In all good examples, replace every $y_t$ with $y'_t = \phi_{x_t}(y_t)$.
% \end{enumerate}
% Let $b$ be the number of bad examples. Consider the sequence $S_g$ which is $S$, without the bad examples. Therefore, $S_g$ and $S'$ are the same sequence with the difference that the $t$'th example in $S$ is $(x_t,y_t)$ and the $t$'th example in $S'$ is $(x_t, y'_t)$. Let us now show that \eqref{eq:reduc-1} holds. For every $h \in \mathcal{H}$, let $\bar{h} \in \bar{\mathcal{H}}$ be the hypothesis obtained by taking $h$ and replacing the label of every $x$ from $h(x)$ to $\phi_x(h(x))$. Therefore, for all $h \in \mathcal{H}$ and for any good example $(x_t,y_t) \in S$ it holds that $h(x_t) = y_t$ if and only if $\bar{h}(x_t) = y'_t$, implying  \eqref{eq:reduc-1}, since all $h \in \mathcal{H}$ are mistaken on all $b$ bad examples. Similarly, to see that \eqref{eq:reduc-2} holds as well, note that for every example $(x_t,y_t) \in S_g$, $\mathcal{A}$ is mistaken on $(x_t,y_t)$ if and only if $\bar{\mathcal{A}}$ is mistaken on $(x_t, y'_t)$, and with the same probability, by the definition of $\mathcal{A}$.

Let us move on to Property (iv). The direction $\operatorname{BL}(\Hcal) \leq \operatorname{BL}(\bar{\Hcal})$ follows from Property (iii). Indeed, let $\bar{\Acal}$ be the optimal BSOA deterministic learner defined in \citep{DanielyERMprinciple} for $\bar{\Hcal}$ under the assumption of realizability. For every round $t$, the algorithm $\bar{\Acal}$ never predicts $y \notin \bar{\Hcal}(x_t)$ by its definition. Therefore, by Property (iii) there exists a deterministic learner $\Acal$ for $\Hcal$ having the same guarantees as of $\bar{\Acal}$. Therefore $\operatorname{BL}(\Hcal) \leq \operatorname{BL}(\bar{\Hcal})$. The reverse direction $\operatorname{BL}(\Hcal) \geq \operatorname{BL}(\bar{\Hcal})$ follows by considering the realizable setting and the bandit algorithm for $\bar{\mathcal{H}}$ that, given any instance $x_t$, passes $x_t$ to the BSOA for $\mathcal{H}$, receives its prediction $\bar{y}_t \in \mathcal{Y}$, makes the prediction $\hat{y}_t = \phi_{x_t}(\bar{y}_t) \in [C]$, and upon receiving the feedback $\mathbbm{1}\{\hat{y}_t \neq y_t\}$, passes the same feedback to the BSOA. The same analysis as in Property (iii) can be used to show that this algorithm makes at most $\text{BL}(\mathcal{H})$ mistakes on any realizable stream. 
% and is more straightforward since $\bar{\Hcal}$ is more restricted than $\Hcal$. Proving it formally is a technicality, and since we do not use it in the paper, we omit the formal proof.
\end{proof}

In order to use Property (iii) of Lemma \ref{lem:simplify}, we need to construct a bandit learner $\bar{\Acal}$ which on any instance $x \in \mathcal{X}$, makes a prediction that lies in $\mathcal{H}(x)$ and achieves a sublinear regret bound whenever $\text{BL}(\mathcal{H}) < \infty$. Unfortunately, the generic bandit learner witnessing the proof of Theorem \ref{thm:finitek} does not guarantee that its predictions always lie in the projection of $\mathcal{H}$. Fortunately, the following lemma, whose proof can be found in Appendix \ref{app:prf}, shows that a slight modification of the bandit learner used to prove Theorem \ref{thm:finitek} can achieve the same regret bound, while ensuring that the predictions always lie in the projection of $\mathcal{H}$. 

\begin{lemma} \label{lem:finitekalgproj}
For any $\mathcal{H} \subseteq \mathcal{Y}^{\mathcal{X}}$, there exists an online learner $\mathcal{A}$ such that 
$$R_{\mathcal{A}}(T, \mathcal{H}) \leq e\sqrt{\emph{\text{L}}(\mathcal{H})|\mathcal{Y}|T\log(T|\mathcal{Y}|)},$$

\noindent while ensuring that $\mathcal{A}(x_t) \in \mathcal{H}(x_t)$ almost surely.
\end{lemma}

We are now ready to prove Theorem \ref{thm:banditagn}, which implies that finitness of BLdim is sufficient for bandit online learnability even when the label space is unbounded. This proves direction $(2) \implies (1)$ in Theorem \ref{thm:necsuff}.

\smallskip

\begin{proof} (of Theorem \ref{thm:banditagn})
    We first prove a stronger result and then show that Theorem \ref{thm:banditagn} follows.   Let $\mathcal{H} \subseteq \mathcal{Y}^{\mathcal{X}}$ be such that $\text{BL}(\mathcal{H}) < \infty$. Then, by Lemmas~\ref{lem:simplify}~and~\ref{lem:finitekalgproj}, there exists an online learner whose expected regret in the agnostic setting under bandit feedback is at most $e\sqrt{\text{L}(\mathcal{H})CT\log(TC)}$ where $C = \sup_{x \in \mathcal{X}}|\mathcal{H}(x)|$. Since Lemma \ref{lem:proj} states that $C \leq \text{BL}(\mathcal{H}) + 1 \leq 2\text{BL}(\mathcal{H})$, we can further upperbound the expected regret by $2e\sqrt{\text{L}(\mathcal{H}) \text{BL}(\mathcal{H})T\log(T\,\text{BL}(\mathcal{H}))}.$ There are now two cases to consider. If $T \leq \text{BL}(\mathcal{H})$, the expected regret of any bandit online learner can be trivially upperbounded by $T$. Noting that $8\sqrt{\text{L}(\mathcal{H})\text{BL}(\mathcal{H})T\log(T)} \geq T$ when $T \leq \text{BL}(\mathcal{H})$ completes this case. If $T > \text{BL}(\mathcal{H})$, then we can upperbound the expected regret of the bandit online learner by 

    $$2e\sqrt{\text{L}(\mathcal{H}) \text{BL}(\mathcal{H})T\log(T\,\text{BL}(\mathcal{H}))} \leq 2e\sqrt{2\text{L}(\mathcal{H}) \text{BL}(\mathcal{H})T\log(T)} \leq 8\sqrt{\text{L}(\mathcal{H}) \text{BL}(\mathcal{H})T\log(T)},$$
    matching the upperbound given in the statement of Theorem \ref{thm:banditagn}. This completes the proof. 
\end{proof}
% The expected regret bound of $e\sqrt{\text{L}(\mathcal{H})CT\log(TC)}$ derived as a part of Theorem \ref{thm:banditagn} can be tight up to logarithmic factors in $T$ by adapting the construction from Section 5 of \cite{auer2002nonstochastic} to the setting where $\mathcal{H} = \mathcal{Y}^{\mathcal{X}}$ for $|\mathcal{X}| < \infty$ and $|\mathcal{Y}| = C$.

In online learning theory, upperbounds on the minimax expected regret are traditionally derived in terms of the single combinatorial dimension that characterizes learnability. However, our upperbound in Theorem \ref{thm:banditagn} is in terms of both the Ldim and BLdim. To get a bound depending only on the BLdim, one can trivially use the fact that $\text{L}(\mathcal{H}) \leq \text{BL}(\mathcal{H})$ to get a suboptimal upperbound of $8 \, \text{BL}(\mathcal{H})\sqrt{T \log(T)}$ on the minimax expected regret. However, as an intermediate step to our upperbound in Theorem \ref{thm:banditagn}, we show that the minimax expected regret can actually be upperbounded by $e\sqrt{\text{L}(\mathcal{H})CT\log(TC)}$, and thus it is natural to ask whether there is an upperbound on $\sqrt{\text{L}(\mathcal{H})C}$ that is significantly better than $\text{BL}(\mathcal{H}).$ Unfortunately, the following example shows that this is not the case. 

\smallskip

\textbf{Example 1}. Fix $d, C \in \mathbb{N}$. Define the instance space $\mathcal{X} = \{x_0, ..., x_d\}$ and the label space $\mathcal{Y} = \{0, ..., C - 1\}$. Let $\mathcal{H}_1 = \{0, 1\}^{\{x_1, ..., x_d\}}$ and $\mathcal{H}_2 = \{x \mapsto y\, \mathbbm{1}\{x = x_0\}: y \in \mathcal{Y}\}$. Consider the hypothesis class $\mathcal{H} = \mathcal{H}_1 \cup \mathcal{H}_2$. Clearly, $\text{L}(\mathcal{H}) \geq \text{L}(\mathcal{H}_1) = \text{BL}(\mathcal{H}_1) = d$. Moreover, $\text{BL}(\mathcal{H}_2) \leq C - 1$. We now give an upperbound on $\text{BL}(\mathcal{H})$ by constructing a deterministic learner for $\mathcal{H}$. Consider the learning algorithm that predicts $0$ until its first mistake, removes inconsistent hypotheses, and plays the Bandit Standard Optimal Algorithm (BSOA) from \cite{DanielyERMprinciple} on future rounds. We now show that this algorithm makes at most $1 + \max\{d, C-1\}$ mistakes on any realizable stream. There are two cases to consider. Suppose the algorithm makes its first mistake on $x_0$. Then, by construction of $\mathcal{H}$, the true hypothesis must be in $\mathcal{H}_2$ and thus the BSOA makes no more than $\text{BL}(\mathcal{H}_2) \leq C - 1$ mistakes in all future rounds. On the other hand, if the algorithm makes its first mistake on $x \in \{x_1, ..., x_d\}$, then the true hypothesis must be in $\mathcal{H}_1$  and thus the BSOA makes at most $\text{BL}(\mathcal{H}_1) = d$ mistakes on all future rounds. Overall, the algorithm makes at most $1 + \max\{d, C-1\}$ mistakes. Since the BLdim lowerbounds the number of mistakes made by any deterministic learner under bandit feedback, we must have that $\text{BL}(\mathcal{H}) \leq 1 + \max\{d, C-1\}$. Taking $C = d + 1$, we have that $\text{BL}(\mathcal{H}) \leq 1 + d \leq 1 + \sqrt{\text{L}(\mathcal{H})C}$, which completes the example.

% The expected regret bound in Theorem \ref{thm:banditagn} can be tight up to polynomial factors of $\text{BL}(\mathcal{H})$, $\text{L}(\mathcal{H})$ and $\log(\text{BL}(\mathcal{H})T)$. \cite{daniely2013price} give a hypothesis class $\mathcal{H}$ for which the lowerbound on expected regret in the agnostic setting is at least $\frac{1}{20}\sqrt{T|\mathcal{Y}|\text{L}(\mathcal{H})}$. Using the fact that  $\frac{\text{BL}(\mathcal{H})}{\text{L}(\mathcal{H})} \leq 4|\mathcal{Y}|\log(|\mathcal{Y}|) \leq 4|\mathcal{Y}|^{\frac{3}{2}}$  gives that the expected regret to learn $\mathcal{H}$ in the agnostic setting is at least $\Omega\left(\sqrt{\text{L}(\mathcal{H})^{\frac{1}{3}} \text{BL}(\mathcal{H})^{\frac{2}{3}} T}\right).$ We leave it as an interesting open question on whether our upperbound can be tight. 
\smallskip

We leave it as an interesting open question to derive optimal lower and upper bounds on the minimax expected regret in terms of only the BLdim (see Section \ref{sec:disc}). Lemma \ref{lem:simplify} can also be used to sharpen the relationship between BLdim and Ldim. In particular, due to \citep{auer1999structural, daniely2013price, long2017new}, there exists a \textit{deterministic} online learner in the realizable setting whose number of mistakes, under bandit feedback, is at most $O(\text{L}(\mathcal{H})|\mathcal{Y}|\log(|\mathcal{Y}|))$. Since the BLdim lowerbounds the number of mistakes made by any deterministic online learner in the realizable setting, Lemma \ref{lem:simplify} immediately implies that when $\sup_{x \in \mathcal{X}} |\mathcal{H}(x)| \leq C$, we have $\text{BL}(\mathcal{H}) = O(\text{L}(\mathcal{H}) C\log C)$, proving direction $(3) \implies (2)$ in Theorem \ref{thm:necsuff}.
%Thus, finiteness of $\text{L}(\mathcal{H})$ and $\sup_{x \in \mathcal{X}} |\mathcal{H}(x)|$  also suffices for bandit online learnability, proving the direction $(3) \implies (1)$ in Theorem \ref{thm:necsuff}.
In Section~\ref{sec:necc}, we show that finiteness of both $C$ and $\operatorname{L}(\Hcal)$ is also necessary for learnability (direction $(1) \implies (3)$).

We end this section with Corollary \ref{cor:banditUC}, which shows that SUC is necessary for a hypothesis class to be bandit online learnable. 

\begin{corollary} \label{cor:banditUC}
If $\emph{\text{BL}}(\mathcal{H}) < \infty$, then $\emph{\text{SG}}(\mathcal{H}) = O(\emph{\text{L}}(\mathcal{H})\log(\emph{\text{BL}}(\mathcal{H})))$. 
\end{corollary}

\begin{proof}
    Let $\mathcal{H} \subseteq \mathcal{Y}^{\mathcal{X}}$ such that $\text{BL}(\mathcal{H}) < \infty$. Then by Lemmas~\ref{lem:proj}~and~\ref{lem:simplify}, there exists a class $\bar{\mathcal{H}} \subseteq [\text{BL}(\mathcal{H}) + 1]^{\mathcal{X}}$ such that $\text{L}(\bar{\mathcal{H}}) = \text{L}(\mathcal{H})$
    and $\text{SG}(\bar{\mathcal{H}}) = \text{SG}(\mathcal{H})$. Since $\text{BL}(\mathcal{H}) + 1 < \infty$, Theorem \ref{prop:sgldim}  implies that $\text{SG}(\bar{\mathcal{H}}) = O(\text{L}(\bar{\mathcal{H}})\log(\text{BL}(\bar{\mathcal{H}}))) = O(\text{L}(\mathcal{H})\log(\text{BL}(\mathcal{H})))$. 
\end{proof}

Since $\text{BL}(\mathcal{H}) < \infty$ implies that $\text{L}(\mathcal{H}) < \infty$, Corollary \ref{cor:banditUC} and  Theorem \ref{prop:sguc} taken together prove the first half of Theorem \ref{thm:banditUC}, showing that $\mathcal{H}$ enjoys SUC when $\text{BL}(\mathcal{H}) < \infty$.  Moreover, when $\text{BL}(\mathcal{H}) < \infty$,  Corollary \ref{cor:banditUC} along with Theorem \ref{prop:sgrates} implies a slightly sharper upperbound on the optimal expected regret in the agnostic setting under \textit{full-information} feedback.  

\begin{corollary} \label{cor:improv}
Let $\mathcal{H} \subseteq \mathcal{Y}^{\mathcal{X}}$ such that $\emph{\text{BL}}(\mathcal{H}) < \infty$. Then, there exists an agnostic online learner whose expected regret, under \emph{full-information feedback}, is at most

$$O\left(\sqrt{\emph{\text{L}}(\mathcal{H})T\log(\emph{\text{BL}}(\mathcal{H}))}\right).$$
\end{corollary}

\begin{proof}
    Let $\mathcal{H}$ be such that $\text{BL}(\mathcal{H}) < \infty$. Then, by Corollary \ref{cor:banditUC}, $\text{SG}(\mathcal{H}) = O(\text{L}(\mathcal{H})\log(\text{BL}(\mathcal{H})))$. Also, by Theorem \ref{prop:sgrates}, we have that under full-information feedback, there exists a online learner whose expected regret is at most $O(\sqrt{T \, \text{SG}(\mathcal{H})})$. Combining these two results gives the stated claim.  
\end{proof}

Namely, Corollary \ref{cor:improv} improves upon the upperbound on expected regret given by \cite[Theorem~1]{hanneke2023multiclass} by replacing the $\log(\frac{T}{\text{L}(\mathcal{H})})$ factor with $\log(\text{BL}(\mathcal{H}))$.

\section{Finite BLdim is Necessary for Bandit Online Learnability} \label{sec:necc}
In this section, we complement the results of Section \ref{sec:Bldimsuff}, and deduce that finiteness of BLdim is necessary for bandit online learnability in the realizable setting even when the label space is unbounded. Since agnostic learnability implies realizable learnability, this also implies that finiteness of the BLdim is necessary for agnostic learnability, completing the proof of the direction $(1) \implies (2)$ in Theorem \ref{thm:necsuff}. This will also imply $(1) \implies (3)$, which completes the proof of Theorem~\ref{thm:necsuff}.

\begin{lemma} \label{lem:necc}
    Let $\mathcal{H} \subseteq \mathcal{Y}^{\mathcal{X}}$ and $C = \sup_{x \in \mathcal{X}}|\mathcal{H}(x)|$. Then, for every bandit online learner $\Acal$:
    \begin{enumerate}
        \item There exists a realizable stream with expected regret at least $\frac{\operatorname{BL}(\mathcal{H})}{4C \log C}$ if $T \geq \operatorname{L}(\Hcal)$ and at least $T/2$ otherwise.
        \item There exists a realizable stream with expected regret at least $\frac{C -1}{2}$ if $T \geq C$, and at least $\frac{T-1}{2}$ otherwise.
    \end{enumerate}
    % $$\max\left\{\frac{\min\{\emph{\text{BL}}(\mathcal{H}), T\}}{4C \log C}, \frac{\min\{C, T\}-1}{2}\right\}.$$
\end{lemma}

\begin{proof}
    %We need to show that for every bandit learner $\mathcal{A}$, and for any expression in the maximization in the statement of the lemma, there exists a realizable input sequence $S$ of length $T$ such that the expected regret of $\mathcal{A}$ is at least as the expression. 
    Let us start with the first item. A well-known result by \cite{ben2009agnostic} states that there exists a realizable stream of length $T = \operatorname{L}(\Hcal)$ such that in expectation, $\mathcal{A}$ makes at least $\operatorname{L}(\mathcal{H})/2$ mistakes under full information feedback. On the other hand, by \cite{long2017new} and Lemma~\ref{lem:simplify} we have $\text{BL}(\mathcal{H}) \leq 2 \text{L}(\mathcal{H}) C \log C$, implying the item for the case $T \geq \operatorname{L}(\Hcal)$. If $T < \operatorname{L}(\Hcal)$, we employ the lower bound on $T$ instead of on $\operatorname{L}(\Hcal)$, concluding this item. The second item follows immediately from \cite[Claim 2]{daniely2013price}.
\end{proof}

Lemma \ref{lem:necc} implies that finiteness of BLdim is  necessary for bandit online learnability in the realizable setting. Recall that $\operatorname{BL}(\Hcal) \geq C-1$ due to Lemma~\ref{lem:proj}. Now, if $C= \infty$ (where $C := \sup_{x\in \mathcal{X}} |\mathcal{H}(x)|$), then Lemma \ref{lem:necc} implies that the expected regret of any online learner under bandit feedback and in the realizable setting, is at least $\frac{T-1}{2}$, a linear function of $T$.  On the other hand, if $\text{BL}(\mathcal{H}) = \infty$ and $C < \infty$, then the bound $\operatorname{BL}(\mathcal{H}) = O(\operatorname{L}(\mathcal{H}) C \log C)$ implies that $\operatorname{L}(\Hcal) = \infty$, and then Lemma~\ref{lem:necc} implies a lowerbound of $\frac{T}{2}$ on the expected regret. This proves the direction $(1) \implies (2)$ in Theorem \ref{thm:necsuff}.  Using the fact that $\operatorname{BL}(\mathcal{H}) \geq \operatorname{L}(\mathcal{H})$ and Lemma \ref{lem:proj} shows that $(2) \implies (3)$, completing the proof of Theorem \ref{thm:necsuff}.

Furthermore, if $C$ is a constant, then taken together with Theorem \ref{thm:reallearn}, Lemma \ref{lem:necc} implies that the BLdim characterizes the optimal expected mistake bound of randomized learners in the realizable setting up to constant factors. In the agnostic setting, the full-information lowerbound of $\sqrt{\frac{\text{L}(\mathcal{H})T}{8}}$ on the expected regret can also be a tight lowerbound under bandit feedback up to logarithmic factors in $T$. For example, for every class $\mathcal{H} \subseteq \mathcal{Y}^{\mathcal{X}}$ such that $\sup_{x \in \mathcal{X}}|\mathcal{H}(x)| \leq 2$, Theorem \ref{thm:finitek} and Lemma \ref{lem:simplify} imply the existence of a bandit online learner whose expected regret is at most $8\sqrt{\text{L}(\mathcal{H})T\log(T)}$.

% \idan{I respected the following characterization, mentioned in a few different places in the paper, with a corollary. If you don't like it, we can remove it, of course.}
% We summarize the conditions for bandit online learnability in the following corollary, implied by Lemma~\ref{lem:necc} and the bound $\operatorname{BL}(\mathcal{H}) = O(\operatorname{L}(\mathcal{H}) C \log C)$.

% \begin{corollary}
%     Let $\Hcal \subset \Ycal^\Xcal$ be a hypothesis class, and let $C := \sup_{x\in \mathcal{X}} |\mathcal{H}(x)|$. The following conditions are equivalent:
%     \begin{enumerate}
%         \item $\Hcal$ is bandit online learnable.
%         \item $\operatorname{BL}(\Hcal) < \infty$.
%         \item $C < \infty$ and $\operatorname{L}(\Hcal) < \infty$.
%     \end{enumerate}
% \end{corollary}

% \smallskip

 Finally, Lemma \ref{lem:necc} together with Lemma \ref{lem:gap}  shows that neither the finitness of Ldim nor the finiteness of SGdim is sufficient for bandit online learnability.

\begin{lemma} \label{lem:gap}
    Let $\Xcal = \{0\}$, $\mathcal{Y} = \mathbb{N}$ and $\mathcal{H} = \{h_a: a \in \mathbb{N}\}$ where $h_a(0) = a$. Then, $\emph{\text{L}}(\mathcal{H}) = \emph{\text{SG}}(\mathcal{H}) = 1$ but $\emph{\text{BL}}(\mathcal{H}) = \infty$. 
\end{lemma}

\begin{proof}
     The equality $\text{BL}(\mathcal{H}) = \infty$ follows from the fact that $|\mathcal{H}(0)| = \infty$ and Lemma \ref{lem:proj}.  We have $\text{L}(\mathcal{H}) = 1$ because for any labeled instance $(0, y) \in \mathcal{X} \times \mathcal{Y}$,  there only exists one hypothesis $h \in \mathcal{H}$ such that $h(0) = y$ (namely $h_y$).  Lastly, $\text{SG}(\mathcal{H}) = 1$ because for any labeled instance $(0, y) \in \mathcal{X} \times \mathcal{Y}$ there exists only one function in the loss class  $\{(0, y) \mapsto \mathbbm{1}\{h(0) \neq y\}: h \in \mathcal{H}\}$ that achieves loss $0$. 
\end{proof}

Lemma \ref{lem:gap} completes the proof of Theorem \ref{thm:banditUC}, since we have exhibited a class for which SUC holds but is not bandit online learnable.

\section{Discussion and Open Questions}\label{sec:disc}
In this paper, we revisited multiclass online learnability under bandit feedback and showed that, when $\mathcal{Y}$ is unbounded: (1) the Bandit Littlestone dimension, originally proposed by \cite{DanielyERMprinciple}, continues to characterize bandit online learnability, and (2) while SUC is necessary for bandit online learnability, it is not sufficient.

Moving forward, there are still many interesting open questions.
By Theorem \ref{thm:banditagn}, in the agnostic setting there is a gap of $\sqrt{\text{BL}(\mathcal{H})\log(T)}$ between the upper and lowerbounds on the optimal expected regret under bandit feedback. Is this gap between the upper and lowerbound unavoidable? Using the fact that $\text{BL}(\mathcal{H}) \leq 4C\log(C)\text{L}(\mathcal{H})$, one can get a lowerbound of $\Omega\left(\sqrt{\frac{\text{BL}(\mathcal{H}) \,T}{C \log(C)}}\right)$ on the expected regret in the agnostic setting, where $C = \sup_{x \in \mathcal{X}}|\mathcal{H}(x)|$. Is it possible to remove the dependence on $C$ and improve this lowerbound to $\Omega(\sqrt{\text{BL}(\mathcal{H}) \, T})$?  

% In proving Theorem \ref{thm:banditagn}, we first showed the existence of an online learner with expected regret at most $O(\sqrt{C \, \text{L}(\mathcal{H})\, T\log(T)})$, where $C = \sup_{x \in \mathcal{X}} |\mathcal{H}(x)|$. We trivially upperbound $C \, \text{L}(\mathcal{H}) \leq  2\text{BL}(\mathcal{H}) \text{L}(\mathcal{H})$ using Lemma \ref{lem:proj}. Is it possible to give a sharper upperbound on the quantity $C \, \text{L}(\mathcal{H})$? Or more generally, can our upperbound on    

While the BLdim provides a sharp quantitative characterization of deterministic learnability in the realizable setting, it is unclear whether it provides a tight \textit{quantitative} characterization of randomized learnability in both the realizable and agnostic settings. Recently, \cite*{filmus2023optimal} gave a combinatorial parameter called the Randomized Littlestone dimension and showed that it exactly quantifies the optimal expected mistake bound for randomized learners in the realizable setting under full-information feedback. Is there a modification of this dimension that can exactly quantify the optimal expected mistake bound for randomized learners in the realizable setting under \textit{bandit feedback}? Can such a dimension also be used to give a sharper upperbound on the expected regret in the agnostic setting?

% Acknowledgments---Will not appear in anonymized version
\acks{AT acknowledges the support of NSF via grant IIS-2007055. VR acknowledges the support of the NSF Graduate Research Fellowship.}

\bibliography{references}

\appendix

\section{Proof of Lemma \ref{lem:finitekalgproj}} \label{app:prf}
% \crefalias{section}{appendix} % uncomment if you are using cleveref

% \idan{In the middle of the proof the notation $\Gcal$ for the set of experts switched to $E$. I changed it to $\Gcal$ in the entire proof.}

To prove Lemma \ref{lem:finitekalgproj}, we slightly modify the generic agnostic learner witnessing the proof of Theorem~\ref{thm:finitek}. Recall that the agnostic learner in Theorem \ref{thm:finitek} first 
constructs a sufficiently small set of experts $E$ such that for every hypothesis $h \in \mathcal{H}$, there exists an expert $\mathcal{E}_h \in E$ whose predictions exactly match $h$ over the stream. Then, the learner runs the non-mixing version of EXP4 (see Figure 4.1 and Theorem 4.2 in \cite{bubeck2012regret}) with this set of experts $E$ on the stream, for an appropriately chosen learning rate. Unfortunately, in all rounds $t \in [T]$, some of the experts constructed by this learner output predictions lying outside of $\mathcal{H}(x_t)$. Thus, EXP4 with this set of experts does not satisfy the constraint imposed by Lemma \ref{lem:finitekalgproj}, that its predictions on $x_t$ must lie in $\mathcal{H}(x_t)$. To fix this issue, we modify each expert $\mathcal{E} \in E$ such that for every $t \in [T]$, we have that $\mathcal{E}(x_t) \in \mathcal{H}(x_t)$ while still maintaining the property of the expert set of \cite{daniely2013price}: for every $h \in \mathcal{H}$, there exists an expert $\mathcal{E}_h \in E$ that predicts exactly like $h$ over the stream. Our modification is simple:  in contrast to the experts constructed by \cite{daniely2013price}, our experts may predict using the ``covering function" $\phi$ (as defined in \cite{daniely2013price}) only if its value lies in $\Hcal(x_t)$. Running the EXP4 algorithm using this new set of experts gives the claimed regret guarantee. We now formalize this construction. 

% \idan{I added this following short explanation for readers that are familiar with the construction of \cite{daniely2013price}. They don't really need to read the entire proof in order to understand the modification.} The modification we make to the set of experts constructed by \cite{daniely2013price} is simple: we slightly change the prediction strategy of every expert. In our construction, in contrast with the construction of \cite{daniely2013price}, an expert may predict using the ``covering function" $\phi$ (as defined in \cite{daniely2013price}) only if its value lies in $\Hcal(x_t)$. We now formalize this construction.  

    Let $(x_1, y_1), ..., (x_T, y_T) \in (\mathcal{X} \times \mathcal{Y})^T$ denote the stream of instances to be observed by the learner and $h^{\star} \in \argmin_{h \in \mathcal{H}} \sum_{t=1}^T \mathbbm{1}\{h(x_t) \neq y_t\}$ denote the optimal hypothesis in hindsight. As stated before, our high-level strategy will be to construct a set of experts $E$ and then run EXP4 using $E$ over the stream. Crucially, we will guarantee that $\mathcal{E}(x_t) \in \mathcal{H}(x_t)$ for every $\mathcal{E} \in E$. 

    Given the time horizon $T$, let $L_T = \{L \subset [T]; |L| \leq \text{L}(\mathcal{H})\}$ denote the set of all possible subsets of $[T]$ of size at most $\text{L}(\mathcal{H})$. For every $L \in L_T$, let $\phi : L \rightarrow \mathcal{Y}$ denote a function mapping time points in $L$ to a label in $\mathcal{Y}$. Let $\Phi_L = \mathcal{Y}^{L}$ denote all such functions $\phi$. For each $L \in L_T$ and $\phi \in \Phi_L$, we define an expert $\mathcal{E}_{L, \phi}$.
    % \idan{Edited the following to match the change we have made in Algorithm 2.}
    As presented below in Algorithm \ref{alg:expert}, expert $\mathcal{E}_{L, \phi}$ uses the Standard Optimal Algorithm (SOA) \citep{Littlestone1987LearningQW} to make its prediction in rounds $t$ where $t \notin L$. When $t \in L$, there are two cases. If $\phi(t) \in \mathcal{H}(x_t)$, the expert $\mathcal{E}_{L, \phi}$ uses the function $\phi$ to compute a labeled instance to predict and update the SOA with. Otherwise, the expert chooses an arbitrary label in $\Hcal(x_t)$ to predict and update SOA with.  Let $E = \bigcup_{L \in L_T} \bigcup_{\phi \in \Phi_L} \mathcal{E}_{L, \phi}$ denote the set of all Experts parameterized by subsets $L \in L_T$ and $\phi \in \Phi_L$. Crucially, observe that by definition of SOA, for every time point $t \in [T]$ and expert $\mathcal{E} \in E$, it holds that $\mathcal{E}(x_t) \in \mathcal{H}(x_t)$. Finally, note that $|E| \leq (T|\mathcal{Y}|)^{\text{L}(\mathcal{H})}$.

\begin{algorithm}
\caption{Expert $\mathcal{E}_{L, \phi}$}
\label{alg:expert}
\setcounter{AlgoLine}{0}
\KwIn{Independent copy of SOA}
\For{$t = 1,...,T$} {
    Receive example $x_t$

    Let $\tilde{y}_t =  \text{SOA}(x_t)$

    \uIf{$t \in L$ and $\phi(t) \in \mathcal{H}(x_t)$}{
       Predict $\hat{y}_t = \phi(t)$
    } 
    \uElseIf{$t \in L$ and $\phi(t) \notin \mathcal{H}(x_t)$} {
         Predict arbitrary label $\hat{y}_t \in \mathcal{H}(x_t)$
    }
    \uElse{
     Predict $\hat{y}_t = \tilde{y}_t$
    }

     Update SOA by passing $(x_t, \hat{y}_t)$}

\end{algorithm}

We claim that there exists an expert $\mathcal{E}_{L^{\star}, \phi^{\star}} \in E$ such that $ h^{\star}(x_t) = \mathcal{E}_{L^{\star}, \phi^{\star}}(x_t)$ for all $t \in [T]$. To see this, consider the hypothetical stream of instances labeled by the optimal hypothesis $S^{\star} = (x_1, h^{\star}(x_1)), ..., (x_T, h^{\star}(x_T)).$ Let $L^{\star} = \{t_1, t_2, ... \}$ be the indices on which the SOA algorithm would have made a mistake had it run on $S^*$. By the guarantees of the SOA \citep{Littlestone1987LearningQW}, we have that $|L^{\star}|  \leq \text{L}(\mathcal{H})$. Consider the function $\phi^{\star}: L^{\star} \rightarrow \mathcal{Y}$ such that for all $t \in L^{\star}$, we have $\phi^{\star}(t) = h^{\star}(x_t)$. By construction of $E$, there exists an expert $\mathcal{E}_{L^{\star}, \phi^{\star}} \in E$ parameterized by $L^{\star}$ and $\phi^{\star}$.  We claim that for all $t \in [T]$, we have $\mathcal{E}_{L^{\star}, \phi^{\star}}(x_t) = h^{\star}(x_t)$. This follows by observing that $\mathcal{E}_{L^{\star}, \phi^{\star}}$ predicts and updates its copy of SOA using exactly the stream of instances labeled by $h^{\star}$.  Since by definition of $L^{\star}$ the predictions of SOA match that of $h^{\star}$ outside of $L^{\star}$, we have $\mathcal{E}_{L^{\star}, \phi^{\star}}(x_t) = \text{SOA}(x_t) = h^{\star}(x_t)$ for all $t \notin L^{\star}$. Moreover, for those time points $t \in L^{\star}$, we have that $\mathcal{E}_{L^{\star}, \phi^{\star}}(x_t) = \phi^{\star}(t) = h^{\star}(x_t)$ by definition of $\phi^{\star}(t)$. Thus, for all $t \in [T]$, we have that $\mathcal{E}_{L^{\star}, \phi^{\star}}(x_t) = h^{\star}(x_t)$.

Consider the agnostic online learner $\mathcal{A}$ that runs the non-mixing version of EXP4 (see Fig. 4.1 and Theorem 4.2 in \cite{bubeck2012regret}) using the set of experts $E$ with learning rate $\eta = \sqrt{\frac{\ln |E|}{T|\mathcal{Y}|}}$. By the guarantees of the EXP4 algorithm, it follows that 

\begin{align*}
 \mathbb{E}\left[\sum_{t=1}^T \mathbbm{1} \{\mathcal{A}(x_t) \not = y_t \} \right] &\leq  \inf_{\mathcal{E} \in E}\sum_{t=1}^T \mathbbm{1}\{\mathcal{E}(x_t) \not = y_t\} +  e\sqrt{T |\mathcal{Y}|\ln |E|}\\   
 &\leq \sum_{t=1}^T \mathbbm{1}\{\mathcal{E}_{L^{\star}, \phi^{\star}}(x_t) \not = y_t\} + e\sqrt{\text{L}(\mathcal{H})|\mathcal{Y}|T\ln(T|\mathcal{Y}|)}\\
 &= \sum_{t=1}^T \mathbbm{1}\{h^{\star}(x_t) \not = y_t\} + e\sqrt{\text{L}(\mathcal{H})|\mathcal{Y}|T\ln(T|\mathcal{Y}|)}.
  \end{align*}

Finally, observing that for all $t \in [T]$,  $\cup_{\mathcal{E} \in E}\{\mathcal{E}(x_t)\} \subseteq \mathcal{H}(x_t)$ together with the fact that EXP4 algorithm in Figure 4.1 of \cite{bubeck2012regret} samples a label using a distribution supported only over $\cup_{\mathcal{E} \in E}\{\mathcal{E}(x_t)\}$ ensures that $\mathcal{A}(x_t) \in \mathcal{H}(x_t)$ almost surely (equivalently, the EXP4 algorithm samples an expert $\mathcal{E} \in E$ and uses its prediction). This completes the proof of Lemma \ref{lem:finitekalgproj}.

\end{document}